\documentclass[twoside,11pt]{article}

\usepackage{jmlr2e}

\usepackage{amsmath}
\usepackage{pgf,tikz}
\usetikzlibrary{arrows,automata}
\usepackage{colortbl}
\usepackage{multirow}
\usepackage{stackrel}
\usepackage{multirow}
\usepackage{tabularx}

\newcommand{\yc}{\cellcolor{yellow}}
\newcommand{\opt}{\mathit{OPT}}
\newcommand{\graph}{\mathit{G}}
\newcommand{\graphset}{\mathcal{G}}
\newcommand{\vertex}[1]{V_{#1}}

\newcommand{\vertices}{\mathit{V}}
\newcommand{\edges}{\mathit{E}}

\newcommand{\parents}{\Pi}

\newcommand{\cost}{\mathit{score}}

\newcommand{\problem}{\mathit{\epsilon}\text{BNSL}}

\newcommand{\ceil}[1]{\lceil #1 \rceil}

\newcommand{\varI}{I}

\newcommand{\gobnilp}{GOBNILP}

\newcommand{\score}[2]{\sigma_{#1}({#2})}
\newcommand{\pen}[2]{t_{#1}({#2})}

\jmlrheading{1}{2020}{1-48}{4/00}{10/00}{meila00a}{Zhenyu A. Liao, Charupriya Sharma, James Cussens, and Peter van Beek}

\ShortHeadings{Credible Bayesian Network Structures}{Liao, Sharma, Cussens, and van Beek}
\firstpageno{1}

\begin{document}

\title{Learning All Credible Bayesian Network Structures for Model Averaging}

\author{\name Zhenyu A. Liao \email z6liao@uwaterloo.ca\\ 
\name Charupriya Sharma \email c9sharma@uwaterloo.ca\\ 
\addr David R. Cheriton School of Computer Science\\ 
University of Waterloo\\ 
Waterloo, ON N2L 3G1, Canada
\AND
\name James Cussens \email james.cussens@bristol.ac.uk\\
\addr Department of Computer Science\\
University of Bristol\\
Bristol, BS8 1QU, United Kingdom
\AND
\name Peter van Beek \email vanbeek@uwaterloo.ca\\
\addr David R. Cheriton School of Computer Science\\ 
University of Waterloo\\ 
Waterloo, ON N2L 3G1, Canada}

\editor{Editor}

\maketitle

\begin{abstract}%
A Bayesian network is a widely used probabilistic graphical model with
applications in knowledge discovery and prediction.  Learning a
Bayesian network (BN) from data can be cast as an optimization
problem using the well-known score-and-search approach.
However, selecting a single model (i.e., the best scoring BN)
can be misleading or may not achieve the best possible accuracy.
An alternative to committing to a single model is to perform
some form of Bayesian or frequentist model averaging, where
the space of possible BNs is sampled or enumerated in some
fashion. Unfortunately, existing approaches for model averaging
either severely restrict the structure of the Bayesian network
or have only been shown to scale to networks with fewer than 30 random
variables. In this paper, we propose a novel approach to model
averaging inspired by performance guarantees in approximation
algorithms. Our approach has two primary advantages. First,
our approach only considers \emph{credible} models in that they
are optimal or near-optimal in score. Second, our approach
is more efficient and scales to significantly larger
Bayesian networks than existing approaches.

\end{abstract}


\section{Introduction}\label{sec:intro}

A Bayesian network is a widely used probabilistic graphical
model with applications in knowledge discovery, explanation,
and prediction~\citep{Darwiche09,KollerF09}.
A Bayesian network (BN) can be learned from data
using the well-known \emph{score-and-search} approach, where a
scoring function is used to evaluate the fit of a proposed BN
to the data, and the space of directed acyclic graphs (DAGs)
is searched for the best-scoring BN.
However, selecting a single model (i.e., the best-scoring BN)
may not always be the best choice.
When one is using BNs for knowledge discovery and explanation
with limited data, selecting a single model may be
misleading as there may be many other BNs that have scores that
are very close to optimal and the posterior probability of
even the best-scoring BN is often close to zero. As well, when one is using
BNs for prediction, selecting a single model may not achieve
the best possible accuracy. 

An alternative to committing to a single model is to
perform some form of Bayesian or frequentist model averaging ~\citep{Claeskens2008,HoetingMRV1999,KollerF09}.
In the context of knowledge discovery, Bayesian model averaging
allows one to estimate, for example, the posterior probability
that an edge is present, rather than just knowing whether the
edge is present in the best-scoring network.
Previous work has proposed
Bayesian and frequentist model averaging approaches to network structure
learning that enumerate the space of all possible DAGs~\citep{KoivistoS04}, sample from the space of all
possible DAGs~\citep{HeTW2016,MadiganR1994}, consider
the space of all DAGs consistent with a given ordering of
the random variables~\citep{Buntine91,DashC2004},
consider the space of tree-structured or other restricted
DAGs~\citep{MadiganR1994,Meila2000}, and consider
only the $k$-best scoring DAGs for some given value of $k$~\citep{ChenCD2015,ChenCD2016,chen2018pruning,ChenT2014,HeTW2016,TianHR10}.
Unfortunately, these existing approaches either severely
restrict the structure of the Bayesian network, such as
only allowing tree-structured networks or only considering a
single ordering, or have only been shown to scale to small
Bayesian networks with fewer than 30 random variables.

In this paper, we propose a novel approach to
model averaging for BN structure learning that is inspired
by performance guarantees in approximation algorithms. Let
$\opt$ be the score of the optimal BN and assume 
without loss of generality that the optimization problem
is to find the minimum-score BN.
Instead of
finding the $k$-best networks for some fixed value of $k$,
we propose to find all Bayesian networks $\mathcal{G}$ that
are within a factor $\rho$ of optimal; i.e.,
\begin{equation}
\label{EQUATION:factor}
\opt \le \cost(\mathcal{G}) \le \rho \cdot \opt,
\end{equation}
for some given value of $\rho \ge 1$, or equivalently,
\begin{equation}
\label{EQUATION:factor_add}
\opt \le \cost(\mathcal{G}) \le \opt+\epsilon,
\end{equation}
for $\epsilon = (\rho - 1) \cdot \opt$. Instead of choosing arbitrary 
values for $\epsilon$, $\epsilon \ge 0$, we show that for the 
two scoring functions BIC/MDL and BDeu, a good choice for the value of $\epsilon$ is 
closely related to the Bayes factor, a model selection criterion 
summarized in~\citep{kass1995bayes}.

Our approach has two primary advantages. First, our approach
only considers \emph{credible } models in that they are
optimal or near-optimal in score. Approaches that enumerate
or sample from the space of all possible models consider DAGs
with scores that can be far from optimal; for example,
for the BIC/MDL scoring function the ratio of worst-scoring
to best-scoring network can be four or five orders of
magnitude\footnote{\citet{MadiganR1994}
deem such models \emph{discredited} when they make a similar
argument for not considering models whose probability is
greater than a factor from the most probable.}. A
similar but more restricted case can be made against the
approach which finds the $k$-best networks since there is no 
\emph{a priori} way to know how to set the parameter $k$ such that only
credible  networks are considered. Second, and perhaps most
importantly, our approach is significantly more efficient and
scales to Bayesian networks with almost 60 random variables. Existing
methods for finding the optimal Bayesian network structure~\citep[see e.g.,][]{BartlettC13,vanBeek2015}
rely heavily for their success on a significant body of
pruning rules that remove from consideration many candidate
parent sets both before and during the search. We show that many 
of these pruning rules can be
naturally generalized to preserve the Bayesian networks that are 
within a factor of optimal. 
We modify GOBNILP~\citep{BartlettC13}, a state-of-the-art method for finding
an optimal Bayesian network, to implement our generalized pruning rules
and to find all \emph{near}-optimal networks. We show in an experimental
evaluation that the modified GOBNILP scales to significantly
larger networks without resorting to restricting the structure of the Bayesian networks that are learned.

\section{Background}\label{sec:background}

In this section, we briefly review the necessary background in
Bayesian networks and scoring functions, and define the Bayesian network structure
learning problem~\citep[for more background on these topics see][]{Darwiche09,KollerF09}.

\subsection{Bayesian Networks}

A Bayesian network (BN) is a probabilistic graphical model that
consists of a labeled directed acyclic graph (DAG), $\graph = (\vertices, \edges)$ in which the vertices $\vertices = \{V_{1}, \ldots, V_{n}\}$
correspond to $n$ random variables, the edges $E$ represent direct
influence of one random variable on another, and each vertex
$\vertex{i}$ is labeled with a conditional probability
distribution $P(V_{i} \mid \Pi_{i})$ that
specifies the dependence of the variable $V_{i}$ on its
set of parents $\Pi_i$ in the DAG. A BN can
alternatively be viewed as a factorized representation of the
joint probability distribution over the random variables and
as an encoding of the Markov condition on the nodes; i.e., given its parents, every variable is conditionally independent of its non-descendants.

Each random variable $V_i$ has state space $\Omega_i$ = $\{v_{i1}$, \ldots , $v_{i{r_i}}\}$\footnote{Our method works with continuous and mixed BNs, although the discussion focuses on the discrete case.}, where $r_i$ is the cardinality of $\Omega_i$ and typically $r_i\geq 2$. Each $\parents_{i}$ has state space $\Omega_{\parents_{i}} = \{\pi_{i1},\ldots,\pi_{i{r_{\parents_{i}}}}\}$.  We use $r_{\parents_{i}}$ to refer to the number of possible instantiations of the parent set $\parents_{i}$ of $V_i$ (see Figure \ref{fig:notation}). 
The set $\theta = \{ \theta_{ijk} \}$ for all $ i = \{1,\ldots, n \}, j = \{1,\ldots,r_{\parents_{i}}\}$ and  $k = \{1,\ldots,r_i\}$ represents parameter estimates in $G$ obtained either from expert knowledge or from a dataset, where each $\theta_{ijk}$ estimates the conditional probability $P(V_{i} = v_{ik} \mid \Pi_{i} = \pi_{ij})$.

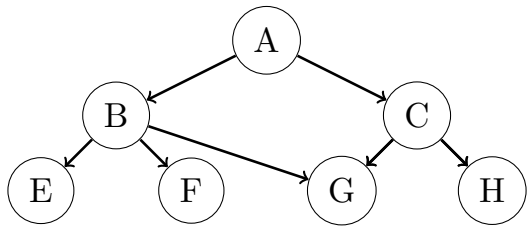
\begin{figure}[htbp]
\begin{center}
\begin{tikzpicture}[every node/.style={circle, draw, scale=1.2, fill=gray!50}, scale=1.0, rotate = 180, xscale = -1]

\node [fill=white](4) at ( 7,6) {E};
\node [fill=white](5) at ( 9,6) {F};
\node [fill=white](6) at ( 8,5) {B};
\node [fill=white](7) at ( 11,6) {G};
\node [fill=white](8) at ( 13,6) {H};
\node [fill=white](9) at ( 12,5) {C};
\node [fill=white](10) at ( 10,4) {A};

\draw [<-, line width=1pt] (4) -- (6) ;
\draw [<-, line width=1pt] (7) -- (6) ;
\draw [<-, line width=1pt] (5) -- (6) ;
\draw [<-, line width=1pt] (7) -- (9) ;
\draw [<-, line width=1pt] (8) -- (9);
\draw [<-, line width=1pt] (7) -- (9) ;
\draw [<-, line width=1pt] (8) -- (9);
\draw [<-, line width=1pt] (6) -- (10) ;
\draw [<-, line width=1pt] (9) -- (10) ;

\end{tikzpicture}
\caption{Example directed acyclic graph of a Bayesian network: Variables $A,B,F$ and $G$ have the state space $\{0,1\}$. The variables $C$ and $E$ have state space $\{0,1,3\} $ and $H$ has state space $\{2,4\}$ 
	Thus $r_A = r_B = r_F = r_G = 2$, $r_C= r_E = 3$ and $r_H = 2$.  Consider the parent set of $G$, $\parents_G = \{B,C\}$ The state space of $\parents_G$ is $\Omega_{\parents_G} = \{ \{0,0\}, \{0,1\}, \{0,3\}, \{1,0\}, \{1,1\}, \{1,3\} \}.$ and $r_{\parents_G} = 6$. }
\label{fig:notation}
\end{center}
\end{figure}

The predominant method for Bayesian network structure learning (BNSL) from data is
the \emph{score-and-search} method.
Let 
$I =
\{I_1, \ldots, I_N\}$ be a dataset where each instance $I_i$ is an $n$-tuple that is a
complete instantiation of the variables in $\vertices$. A \emph{scoring
function} $\sigma( \graph \mid I )$ assigns a real value measuring 
the quality of $\graph=(\vertex,\edges)$ given the data $I$.
Without loss of generality, we assume that a lower score
represents a better quality network structure and omit $I$ when the data is clear from context. 

\begin{definition}[credible network]
Given a non-negative constant $\epsilon$ and a dataset $I = \{I_1, \ldots, I_N\}$, a \textbf{credible  network} $G$ is a network that has a score $\score{}{\graph}$ such that $\opt \leq \score{}{\graph} \leq \opt + \epsilon$, where $\opt$ is the score of the optimal Bayesian network.  
\end{definition}

In this paper, we focus on solving a problem we call the 
$\epsilon$-Bayesian Network Structure Learning ($\problem$). Note 
that the BNSL for the optimal network(s) is a special case of 
$\problem$ where $\epsilon=0$.

\begin{definition}[$\problem$]
Given a non-negative constant $\epsilon$, a dataset $I = \{I_1, \ldots, I_N\}$ 
over random variables $\vertices$ = $\{V_{1}$, \ldots, $V_{n}\}$ and a scoring 
function $\sigma$, the $\epsilon$-Bayesian Network Structure Learning ($\problem$) 
problem is to find all credible networks.
\end{definition}

\subsection{Scoring Functions}

Scoring functions usually balance goodness of fit to
the data with a penalty term for model complexity
to avoid overfitting. Common scoring functions
include BIC/MDL~\citep{LamB94,Schwarz78} and BDeu~\citep{Buntine91,HeckermanGC95}. An important property of these
(and most) scoring functions is decomposability,
where the score of the entire network
$\sigma( \graph )$
can be rewritten as the sum of local scores associated to each vertex
$\sum_{i = 1}^{n} \sigma(\vertex{i},\parents_i )$
that only depends on $\vertex{i}$ and its parent set $\parents_{i}$ in
$\graph$. The local score is abbreviated below as $\score{}{\parents_i}$  when the local node $\vertex{i}$ is clear from context.
Pruning techniques can be used to reduce the
number of candidate parent sets that need to be considered,
but in the worst-case the number of candidate parent sets
for each variable $\vertex{i}$ is exponential in $n$, where $n$
is the number of vertices in the DAG.

In this work, we focus on the Bayesian Information Criterion (BIC) and the Bayesian Dirichlet, specifically BDeu, scoring functions. The BIC scoring function\footnote{We adopt the MDL notation that calculates a positive score.} in this paper is defined as,
\begin{equation*}
    BIC : \score{}{\graph} = -\max_{\theta} L_{G,I} (\theta) + t({G})\cdot w = - \sum_{i=1}^n \sum_{j=1}^{r_{\Pi_i}} \sum_{k=1}^{r_i} n_{ijk} \log \frac{n_{ijk}}{n_{ij}} +  \sum_{i=1}^n r_{\Pi_i}(r_i -1)\frac{\log N}{2}.
\end{equation*}
Here, $w = \frac{\log N}{2}$, $t(G)$ is a penalty term and $L_{G,I} (\theta)$ is the log likelihood, given by,
\begin{equation*}
    L_{G,I} (\theta) = \displaystyle  \sum_{i=1}^n \sum_{j=1}^{r_{\Pi_i}} \sum_{k=1}^{r_i} \log \theta_{ijk}^{n_{ijk}} ,
\end{equation*}
where $n_{ijk}$ is the number of instances in $I$ where $v_{ik}$ and $\pi_{ij}$ co-occur.
As the BIC function is decomposable, we can associate a score to $\Pi_i$, a candidate parent set of $\vertex{i}$ as follows,
\begin{equation*}
    BIC : \score{}{\Pi_i} = -\max_{\theta_i} L(\theta_i) + t({\Pi_i})\cdot w = - \sum_{j=1}^{r_{\Pi_i}} \sum_{k=1}^{r_i} n_{ijk} \log \frac{n_{ijk}}{n_{ij}} +  r_{\Pi_i}(r_i -1)\frac{\log N}{2}.
\end{equation*}
Here, $L (\theta_i) =\sum_{j=1}^{r_{\Pi_i}}\sum_{k=1}^{r_i}  \log \theta_{ijk}^{n_{ijk}}$ and $t(\Pi_i) = r_{\Pi_i}(r_i -1)$.
The BDeu scoring function\footnote{Our BDeu notation calculates a positive score that is consistent with the minimization setting.} in this paper is defined as,
\begin{align*}
    BDeu : \score{}{G} &= \displaystyle  - \sum_{i=1}^n \left( \sum_{j=1}^{r_{\Pi_i}} \left( \log \frac{\Gamma \left( \frac{\alpha}{r_{\parents_i}} \right)}{ \Gamma \left( \frac{\alpha}{r_{\parents_i}} + n_{ij} \right)} + \sum_{k=1}^{r_i} \log \frac{\Gamma \left(\frac{\alpha}{r_i r_{\parents_i}}+ n_{ijk}\right)}{\Gamma\left(\frac{\alpha}{r_i r_{\parents_i}} \right)} \right) \right),
\end{align*}
where $\alpha$ is the equivalent sample size and $n_{ij} = \sum_k n_{ijk}$.
As the BDeu function is decomposable, we can associate a score to $\Pi_i$, a candidate parent set of $\vertex{i}$ as follows,
\begin{align*}
    BDeu : \score{}{\Pi_i} &= \displaystyle  - \sum_{j=1}^{r_{\Pi_i}} \left( \log \frac{\Gamma \left( \frac{\alpha}{r_{\parents_i}} \right)}{ \Gamma \left( \frac{\alpha}{r_{\parents_i}} + n_{ij} \right)} + \sum_{k=1}^{r_i} \log \frac{\Gamma \left(\frac{\alpha}{r_i r_{\parents_i}}+ n_{ijk}\right)}{\Gamma\left(\frac{\alpha}{r_i r_{\parents_i}} \right)} \right).
\end{align*}

\section{The Bayes Factor}\label{sec:bf}

In this section, we show that a good choice for the value of $\epsilon$ for the
$\problem$ problem is closely related to the 
Bayes factor (BF), a model selection criterion summarized in~\citep{kass1995bayes}.

The BF was proposed by Jeffreys as an alternative to significance tests~\citep{jeffreys1967theory}. It was thoroughly examined as a practical model selection tool in~\citep{kass1995bayes}. Let $\graph_0$ and $\graph_1$ be two DAGs (BNs) in the set of all DAGs $\graphset$ defined over a set of random variables $\vertex{}$. The BF in the context of BNs is defined as,
$$
BF(\graph_0,\graph_1)=\frac{P(\varI \mid \graph_0)}{P(\varI \mid \graph_1)},
$$
namely the odds of the probability of the data predicted by network $\graph_0$ and $\graph_1$. The actual calculation of the BF often relies on  Bayes' Theorem as follows,
$$
\frac{P(\graph_0 \mid \varI)}{P(\graph_1 \mid \varI)}=\frac{P(\varI \mid \graph_0)}{P(\varI \mid \graph_1)}\cdot\frac{P(\graph_0)}{P(\graph_1)}=\frac{P(\varI,\graph_0)}{P(\varI,\graph_1)}.
$$
Since it is typical to assume the prior over models is uniform, the BF can then be obtained using either $P(\graph \mid \varI)$ or $P(\varI,\graph), \forall\graph\in\graphset$. We use those two representations to show how BIC and BDeu scores relate to the BF.

Using the Laplace approximation and other simplifications, \citet{ripley1996pattern} derived the following approximation to the logarithm of the marginal likelihood for network $\graph$~\citep[a similar derivation is given in][]{Claeskens2008},
\begin{align*}
\log{P(\varI \mid \graph)} =& \max_{\theta} L_{G,I} (\theta) - t({G})\cdot \frac{\log N}{2}+t({G}) \cdot \frac{\log{2\pi}}{2}\\
&-\frac{1}{2}\log{|J_{G,I} (\theta)|} + \log{P(\theta \mid \graph)},
\end{align*}
where $J_{G,I} (\theta)$ is the Hessian matrix evaluated at the maximum likelihood estimate. It follows that,
$$
\log{P(\varI \mid \graph)}=-BIC(\varI,\graph)+O(1).
$$
The above equation shows that the BIC score was designed to approximate the log marginal likelihood. 
If we drop the lower-order term, we can then obtain the following equation,
\begin{align*}
    BIC(\varI,\graph_1)-BIC(\varI,\graph_0)&=\log{\frac{P(\varI \mid \graph_0)}{P(\varI \mid \graph_1)}} = \log{BF(\graph_0,\graph_1)}.
\end{align*}

It has been indicated in~\citep{kass1995bayes} that as $N\rightarrow\infty$, the difference of the two BIC scores, 
dubbed the Schwarz criterion,  approaches the true value of $\log{BF}$ such that,
$$
\frac{BIC(\varI,\graph_1)-BIC(\varI,\graph_0)-\log{BF(\graph_0,\graph_1)}}{\log{BF(\graph_0,\graph_1)}}\rightarrow 0.
$$
Therefore, the difference of two BIC scores can be used as a rough approximation to $\log{BF}$. Note that some papers define BIC to be twice as large as the BIC defined in this paper, but the above relationship still holds albeit with twice the logarithm of the BF.

Similarly, the difference of the BDeu scores can be expressed in terms of the BF. In fact, the BDeu score is the log marginal likelihood where there are Dirichlet distributions over the parameters~\citep{Buntine91,HeckermanGC95}; i.e.,
$$
\log{P(\varI,\graph)}=-BDeu(\varI,\graph),
$$
and thus,
\begin{align*}
    BDeu(\varI,\graph_1)-BDeu(\varI,\graph_0)&=\log{\frac{P(\varI,\graph_0)}{P(\varI,\graph_1)}} = \log{BF(\graph_0,\graph_1)}.
\end{align*}

The above results are consistent with the observation in~\citep{kass1995bayes} that the $\log{BF}$ can be interpreted as a measure for the \emph{relative success} of two models at predicting data, sometimes referred to as the ``weight of evidence'', without assuming either model is true. The maximal acceptable distance from the optimal model, however, is often specific to a study and determined with domain knowledge; e.g., a BF of 1000 is more appropriate in forensic science. \citet{HeckermanGC95} proposed the following interpreting scale for the BF: a BF of 1 to 3 bears only anecdotal evidence, a BF of 3 to 20 suggests some positive evidence that $\graph_0$ is better, a BF of 20 to 150 suggests strong evidence in favor of $\graph_0$, and a BF greater than 150 indicates very strong evidence. If we deem 20 to be the desired BF in $\problem$, i.e., $\graph_0=\graph^*$ and $\epsilon=\log(20)$, then any network with a score less than $\log(20)$ away from the optimal score would be \emph{credible}, otherwise it would be \emph{discredited}. Note that the ratio of posterior probabilities was defined as $\lambda$ in~\citep{TianHR10,ChenT2014} and was used as a metric to assess arbitrary values of $k$ in finding the $k$-best networks.

Finally, the $\problem$ problem using the BIC or BDeu scoring function given a desired BF can be written as,
\begin{align}\label{EQUATION:bf}
    &\opt \le \cost(\graphset) \le \opt+\log{BF}.
\end{align}

\section{Pruning Rules for Candidate Parent Sets}\label{sec:pruning}

To find all near-optimal BNs given a BF, the local score $\sigma( \parents_{i} )$ for each candidate parent set $\parents_{i} \subseteq 2^{\vertices - \{\vertex{i}\}}$ and each random variable $\vertex{i}$ must be considered. As this is very cost prohibitive, it is important that the search space of candidate parent sets be pruned, provided that global optimality constraints are not violated. In this section, we generalize existing pruning rules such that the generalized rules hold when solving the $\problem$ problem.





A candidate parent set $\Pi_i$ can be \textit{safely pruned} given a non-negative constant $\epsilon \in \mathbb{R}^+$ if $\Pi_i$ cannot be the parent set of $V_i$ in any network in the set of credible networks. Note that for $\epsilon=0$, the set of credible  networks just contains the optimal network(s). We discuss the original rules and their generalization below and proofs for each can be found in the \emph{appendix}.


\citet{TeyssierK05} give a pruning rule for all decomposable scoring functions. This rule compares the score of a candidate parent set to those of its subsets. We give a relaxed version of the rule.
	

	\begin{lemma}
	Given a vertex variable $\vertex{j}$, candidate parent sets
	$\Pi_j$ and  $\Pi_j^{\prime}$, and some $\epsilon\in \mathbb{R}^+$,  if $\Pi_j \subset \Pi_j^{\prime}$ and $\score{}{\Pi_j} + \epsilon \geq \score{}{\Pi_j'}$,
	$\Pi_j^{\prime}$ can be safely pruned if $\sigma$ is a decomposable scoring function. \label{lem:scoreprune}
	\end{lemma}
	
\subsection{Pruning with BIC/MDL Score}


A pruning rule comparing the BIC score and penalty associated to a candidate parent set to those of its subsets was introduced in~\citep{CamposJ11}. The following theorem gives a relaxed version of that rule.
\begin{theorem}
Given a vertex variable $\vertex{j}$,  candidate parent sets
	$\Pi_j$ and $\Pi_j'$, and some $\epsilon \in \mathbb{R}^+$,
	if $\Pi_j \subset \Pi_j'$ and $\score{}{\Pi_j} -  \pen{}{\Pi_j'} + \epsilon < 0$,
	$\Pi_j'$ and all supersets of $\Pi_j^{\prime}$ can be safely pruned  if $\sigma$ is the BIC scoring function. 
	\label{thm:decampospar}
\end{theorem}

Another pruning rule for BIC appears in~\citep{CamposJ11}. This provides a bound on the number of possible instantiations of subsets of a candidate parent set.


	

\begin{theorem}
	Given a vertex variable $V_i$, and a candidate parent set $\Pi_i$ such that $r_{\Pi_i}> \frac{N}{w} \frac{\log r_i}{r_i -1} + \epsilon$ for some $\epsilon \in \mathbb{R}^+$,  if $\parents_i \subsetneq \parents_i'$ , then $\parents_i'$ can be safely pruned if $\sigma$ is the BIC scoring function. \label{THEOREM:decamposbicrelaxed}
\end{theorem}

The following corollary of Theorem \ref{THEOREM:decamposbicrelaxed} gives a useful upper bound on the size of a candidate parent set.





\begin{corollary}
	Given a vertex variable $\vertex{i}$ and candidate parent set 
	$\Pi_i$, if $\Pi_i$ has more than $\ceil{\log_2 (N + \epsilon)}$ elements, for some $\epsilon \in \mathbb{R}^+$, $\Pi_i$ can be safely pruned if $\sigma$ is the BIC scoring function. \label{cor:decampossizerelaxed}
	\end{corollary}
	
Corollary~\ref{cor:decampossizerelaxed} provides an upper-bound on the size of parent sets based solely on the dataset size $N$. The following table summarizes such an upper-bound given different amounts of data $N$ and a BF of 20.
\begin{table}[ht]
    \centering
    \begin{tabular}{@{}c||ccccccc@{}}
     $N$ & $100$ & $500$ & $10^{3}$ & $5\times10^{3}$ & $10^{4}$ & $5\times10^{4}$ & $10^{5}$ \\\hline
     $|\parents|$ & 7 & 9 & 10 & 13 & 14 & 16 & 17
    \end{tabular}
    \label{tab:paLim}
\end{table}

The entropy of a candidate parent set is also a useful measure for pruning. A pruning rule, given by~\citet{Campos2017}, provides an upper bound on the conditional entropy of candidate parent sets and their subsets. We give a relaxed version of their rule. First, we note that the sample estimate of entropy for a variable $\vertex{i}$ is given by,
\begin{align*}
    H(\vertex{i}) &=  -\sum^{r_i}_{k=1}\frac{n_{ik}}{N}\log \frac{n_{ik}}{N} ,
\end{align*}
where $n_{ik}$ represents how many instances in the dataset contain $v_{ik}$, where $v_{ik}$ is an element in the state space $\Omega_i$ of $\vertex{i}$. Similarly, the sample estimate of entropy for a candidate parent set $\Pi_i$ is given by,
\begin{align*}
    H(\Pi_i) &=  -\sum^{r_{\Pi_i}}_{j=1}\frac{n_{ij}}{N}\log \frac{n_{ij}}{N} .
\end{align*}
Conditional entropy is given by,
\begin{equation*}
    H(X \mid Y ) = H(X \cup Y) - H(Y) .
\end{equation*}




\begin{theorem}
Given a vertex variable $V_i$, and candidate parent set $\Pi_i$, let $V_j \notin \Pi_i$  such that $N \cdot \min \{H(V_i \mid \Pi_i), H(V_j \mid \Pi_i)\} \geq (1 - r_{j}) \cdot t(\Pi_i) +\epsilon$ for some $\epsilon \in \mathbb{R}^+$. Then the candidate parent set $\Pi_i' = \Pi_i \cup \{V_j \}$ and all its supersets can be safely pruned if $\sigma$ is the BIC scoring function.
\label{thm:entropyrelaxed}
\end{theorem}

\subsection{Pruning with BDeu Score}

A pruning rule for the BDeu scoring function appears in~\citep{Campos2017} and a more general version is included in~\citep{cussens2015gobnilp}. Here, we present a relaxed version of the rule in~\citep{cussens2015gobnilp}.  

\begin{theorem}
Given a vertex variable $V_i$ and candidate parent sets $\Pi_i$ and $\Pi_i'$ such that $\Pi_i \subset \Pi_i'$ and $\Pi_i \neq \Pi_i'$, let
$r_i^{+}(\Pi_i') := |\{j : n_{ij} > 0 , j \in \Omega_{\Pi_i'}\}|$ be the total number of  instantiations of $\Pi_i'$ that appear in the dataset. If $\score{}{\Pi_i} + \epsilon < r_i^{+}(\Pi_i') \log r_i$, for some $\epsilon \in \mathbb{R^+}$ then $\Pi_i'$ and the supersets of $\Pi_i'$ can be safely pruned if $\sigma$ is the BDeu scoring function.
\label{thm:jamesrelaxed}
\end{theorem}

\section{Experimental Evaluation}\label{sec:exp}

In this section, we evaluate our proposed BF-based method and compare its performance with published $k$-best solvers. 

Our proposed method is more memory efficient comparing to the $k$-best based solvers in BDeu scoring and often collects more networks in a shorter period of time. With the pruning rules generalized above, our method can scale up to datasets with 57 variables in BIC scoring, whereas the previous best results are reported on a network of 29 variables using the $k$-best approach with score pruning~\citep{chen2018pruning}.

The datasets are obtained from the UCI Machine Learning 
Repository~\citep{Dua:2017} and the Bayesian Network 
Repository\footnote{\url{http://www.bnlearn.com/bnrepository/}}.
Both BIC/MDL~\citep{Schwarz78,LamB94} and 
BDeu~\citep{Buntine91,HeckermanGC95} scoring functions are used 
where applicable. All experiments are conducted on computers with 2.2 GHz Intel E7-4850V3 processors. Each experiment is limited to 64 GB of memory and 24 hours of CPU time.

We demonstrate the effect of applying pruning rules during scoring in Section~\ref{sec:pruning_effect}. 
Our generalized rules are able to eliminate the majority of the search space and therefore allow us to apply the $\epsilon$-BNSL algorithm to medium sized networks. 
We discuss the implementation details of the BF approach in Section~\ref{sec:bf_approach} and present experimental results with BIC scores on a wide range of datasets commonly used in BNSL. 
We show the effect of varying sample sizes on our approach using data generated from synthetic BNs in Section~\ref{sec:synth_data}. 
Finally, We compare our approach with the $k$-best method in Section~\ref{sec:bf_v_k}.

\subsection{The Effect of Pruning}\label{sec:pruning_effect}

\begin{table}[!ht]
    \tiny
    \centering
    \begin{tabular}{lrr||rrr||r}
    & & & \multicolumn{3}{c||}{BIC} & BDeu \\
        Data & $n$ & $N$ & $|\parents_{3}|$ & $|\parents_{20}|$ & $|\parents_{150}|$ & $|\parents_{20}|$\\
        \hline
tic tac toe & 10 & 958    & 96      & 110       & 118       & 70     \\
wine        & 14 & 178    & 592     & 949       & 1,582     & 1,256  \\
adult       & 14 & 32,561 & 3,660   & 3,951     & 4,299     & 3,686  \\
nltcs       & 16 & 3,236  & 8,287   & 8,966     & 9,712     & 9,074  \\
msnbc       & 17 & 58,265 & 48,043  & 49,630    & 51,335    & 54,280 \\
letter      & 17 & 20,000 & 117,405 & 120,685   & 124,133   & 87,183 \\
voting      & 17 & 435    & 429     & 497       & 581       & 721    \\
zoo         & 17 & 101    & 1,036   & 1,848     & 3,419     & 28,872 \\
hepatitis   & 20 & 155    & 474     & 1,485     & 4,437     & 4,054  \\
parkinsons  & 23 & 195    & 3,212   & 5,532     & 10,468    & 14,415 \\
sensors     & 25 & 5456   & 962,400 & 1,012,964 & 1,064,961 & OT     \\
autos       & 26 & 159    & 3,413   & 7,629     & 17,442    & 54,511 \\
insurance   & 27 & 1,000  & 530     & 607       & 709       & OT     \\
horse       & 28 & 300    & 760     & 2,296     & 7,361     & OT     \\
flag        & 29 & 194    & 1,227   & 3,888     & 12,873    & OT     \\
wdbc        & 31 & 569    & 17,193  & 23,923    & 34,983    & OT     \\
mildew      & 35 & 1000   & 128     & 128       & 128       & OT     \\
soybean     & 36 & 266    & 7,781   & 14,229    & 29,691    & OT     \\
alarm       & 37 & 1000   & 818     & 1,588     & 4,922     & OT     \\
bands       & 39 & 277    & 1,422   & 5,055     & 19,253    & OT     \\
spectf      & 45 & 267    & 1,320   & 7,407     & 34,971    & OT     \\
sponge      & 45 & 76     & 741     & 1,267     & 3,064     & OT     \\
barley      & 48 & 1000   & 244     & 246       & 256       & OT     \\
hailfinder  & 56 & 100    & 185     & 254       & 452       & OT     \\
hailfinder  & 56 & 500    & 428     & 459       & 519       & OT     \\
lung cancer & 57 & 32     & 567     & 2,392     & 7,281     & OT      
    \end{tabular}
    \caption{The number of candidate parents $|\parents|$ in the pruned scoring file at BF = 3, 20 and 150 using BIC, and at BF = 20 using BDeu, where $n$ is the number of random variables in the dataset, $N$ is the number of instances in the dataset and OT = Out of Time.}\label{TAB:score}
\end{table}

We modified the development version (9c9f3e6) of GOBNILP to apply the generalized pruning rules in Section~\ref{sec:pruning}. In particular, Lemma~\ref{lem:scoreprune} is applied to both BIC and BDeu; Theorem~\ref{thm:decampospar} and Corollary~\ref{cor:decampossizerelaxed} are applied to BIC; Theorem~\ref{thm:jamesrelaxed} is applied to BDeu. The combination of those rules effectively pruned more than $95\%$ of the parent sets for almost all datasets. The worst pruning rate of $88.9\%$ is observed on the letter dataset with a BF of 150 using BIC. We report the number of remaining candidate parent sets in Table~\ref{TAB:score}. The generalized pruning rules allow us to scale up to medium sized networks, unlike previous approaches where the lack of effective pruning rules restricts them to small networks.

\subsection{The Bayes Factor Approach}\label{sec:bf_approach}

\begin{table}[!ht]
    \tiny
    \centering
    \begin{tabular}{lrr||rrr||rrr||rrr}
        Data & $n$ & $N$ & $T_{3}$ (s) & $|\graphset_{3}|$ & $|\mathcal{M}_{3}|$ & $T_{20}$ (s) & $|\graphset_{20}|$ & $|\mathcal{M}_{20}|$ & $T_{150}$ (s) & $|\graphset_{150}|$ & $|\mathcal{M}_{150}|$ \\
        \hline
tic tac toe & 10 & 958 & 1.9 & 192 & 64 & 2.0 & 192 & 64 & 3.3 & 544 & 160 \\
wine & 14 & 178 & 4.1 & 308 & 51 & 24.9 & 3,449 & 576 & 143.7 & 26,197 & 4,497 \\
adult & 14 & 32,561 & 17.5 & 324 & 162 & 45.1 & 1,140 & 570 & 55.7 & 2,281 & 1,137 \\
nltcs & 16 & 3,236 & 53.8 & 240 & 120 & 201.7 & 1,200 & 600 & 1,005.1 & 4,606 & 2,303 \\
msnbc & 17 & 58,265 & 3,483.0 & 24 & 24 & 7,146.9 & 960 & 504 & 8,821.4 & 1,938 & 1,026 \\
letter & 17 & 20,000 & \yc OT & --- & --- & \yc OT & --- & --- & \yc OT & --- & --- \\
voting & 17 & 435 & 1.3 & 27 & 2 & 4.0 & 441 & 33 & 14.3 & 2,222 & 170 \\
zoo & 17 & 101 & 8.1 & 49 & 13 & 21.9 & 1,111 & 270 & 299.3 & 21,683 & 5,392 \\
hepatitis & 20 & 155 & 7.1 & 580 & 105 & 513.3 & 87,169 & 15,358 & 1,452.8 & 150,000 & 49,269 \\
parkinsons & 23 & 195 & 30.7 & 1,088 & 336 & 3,165.9 & 150,000 & 39,720 & 4,534.3 & 150,000 & 116,206 \\
sensors & 25 & 5456 & \yc OT & --- & --- & \yc OT & --- & --- & \yc OT & --- & --- \\
autos & 26 & 159 & 95.0 & 560 & 200 & 2,382.8 & 50,374 & 17,790 & 6,666.9 & 150,000 & 54,579 \\
insurance & 27 & 1,000 & 49.8 & 8,226 & 2,062 & 244.9 & 104,870 & 25,580 & 414.5 & 148,925 & 36,072 \\
horse & 28 & 300 & 18.8 & 1,643 & 246 & 1,358.8 & 150,000 & 28,186 & 1,962.5 & 150,000 & 69,309 \\
flag & 29 & 194 & 16.1 & 773 & 169 & 4,051.9 & 150,000 & 39,428 & 5,560.9 & 150,000 & 122,185 \\
wdbc & 31 & 569 & 396.1 & 398 & 107 & 10,144.2 & 28,424 & 8,182 & 45,938.2 & 150,000 & 54,846 \\
mildew & 35 & 1000 & 1.2 & 1,026 & 2 & 1.2 & 1,026 & 2 & 2.1 & 2,052 & 4 \\
soybean & 36 & 266 & 7,729.4 & 150,000 & 150,000 & 16,096.8 & 150,000 & 62,704 & 8,893.5 & 150,000 & 118,368 \\
alarm & 37 & 1000 & 6.3 & 1,508 & 122 & 684.2 & 123,352 & 9,323 & 2,258.4 & 150,000 & 8,484 \\
bands & 39 & 277 & 100.9 & 7,092 & 810 & 2,032.6 & 150,000 & 44,899 & 16,974.8 & 150,000 & 95,774 \\
spectf & 45 & 267 & 432.4 & 27,770 & 4,510 & 7,425.2 & 150,000 & 51,871 & 19,664.8 & 150,000 & 63,965 \\
sponge & 45 & 76 & 16.8 & 1,102 & 65 & 1,301.0 & 146,097 & 7,905 & 1,254.4 & 150,000 & 90,005 \\
barley & 48 & 1000 & 0.8 & 182 & 1 & 0.8 & 364 & 2 & 1.3 & 1,274 & 5 \\
hailfinder & 56 & 100 & 171.5 & 150,000 & 20 & 149.4 & 150,000 & 748 & 214.6 & 150,000 & 294 \\
hailfinder & 56 & 500 & 286.1 & 150,000 & 30,720 & 314.1 & 150,000 & 18,432 & 217.3 & 150,000 & 24,576 \\
lung cancer & 57 & 32 & 584.3 & 150,000 & 40,621 & 966.6 & 150,000 & 79,680 & 2,739.7 & 150,000 & 48,236 \\

    \end{tabular}
    \caption{The search time $T$, the number of collected networks $|\graphset|$ and the number of MECs $|\mathcal{M}|$ in the collected networks at BF = 3, 20 and 150 using BIC, where $n$ is the number of random variables in the dataset, $N$ is the number of instances in the dataset and OT = Out of Time.}\label{TAB:bf}
\end{table}

We modified the development version (9c9f3e6) of GOBNILP, denoted hereafter as \gobnilp{\_dev}, and supplied appropriate parameter settings for collecting near-optimal networks\footnote{The modified code is available at: \url{https://www.cs.york.ac.uk/aig/sw/gobnilp/}}. The code is compiled with SCIP 6.0.0 and CPLEX 12.8.0. \gobnilp{} extends the SCIP Optimization Suite~\citep{GleixnerEtal2018OO} by adding a \emph{constraint handler} for
handling the acyclicity constraint for DAGs. If multiple BNs are
required \gobnilp{\_dev} just calls SCIP to ask it to collect feasible
solutions.  In this mode, when SCIP finds a solution, the solution is stored, a
constraint is added to render that solution infeasible and the search
continues. This differs from (and is much more efficient than)
the method used in the current stable version of \gobnilp{} for finding $k$-best BNs where an entirely
new search is started each time a new BN is found. A recent version of
SCIP has a separate ``reoptimization'' method which might allow better
$k$-best performance for \gobnilp{} but we do not explore that here.
By default when SCIP is asked to collect solutions it turns off all
cutting plane algorithms. This led to very poor \gobnilp{} performance
since \gobnilp{} relies on cutting plane generation. Therefore, this default setting is overridden in \gobnilp{\_dev} to allow cutting planes when collecting solutions.
To find only solutions with objective no worse than ($\opt + \epsilon$), SCIP's
\texttt{SCIPsetObjlimit} function is used. Note that, for efficiency
reasons, this is \textbf{not} effected by adding a linear constraint.

\begin{table}[!ht]
    \tiny
    \centering
    \begin{tabular}{lrr||rr||rr||rrr}
        Data & $n$ & $N$ & $T_k$ (s) & $k$ & $T_{EC}$ (s) & $|\graphset_k|$ & $T_{20}$ (s) & $|\graphset_{20}|$ & $|\mathcal{M}_{20}|$\\
        \hline
        \multirow{3}{*}{ tic tac toe } & \multirow{3}{*}{ 10 } & \multirow{3}{*}{ 958 } & 0.2 & 10 & 0.5 & 67 & \multirow{3}{*}{ 0.6 } & \multirow{3}{*}{ 152 } & \multirow{3}{*}{ 24 } \\
& & & 2.8 & 100 & 6.0 & 673 & & & \\
& & & 70.7 & 1,000 & 78.5 & 7,604 & & & \\\hline
\multirow{3}{*}{ wine } & \multirow{3}{*}{ 14 } & \multirow{3}{*}{ 178 } & 3.4 & 10 & 12.0 & 60 & \multirow{3}{*}{ 35.9 } & \multirow{3}{*}{ 8,734 } & \multirow{3}{*}{ 6,262 } \\
& & & 85.0 & 100 & 168.4 & 448 & & & \\
& & & 3,420.4 & 1,000 & 3,064.4 & 4,142 & & & \\\hline
\multirow{3}{*}{ adult } & \multirow{3}{*}{ 14 } & \multirow{3}{*}{ 32,561 } & 3.3 & 10 & 633.5 & 68 & \multirow{3}{*}{ 9.3 } & \multirow{3}{*}{ 792 } & \multirow{3}{*}{ 19 } \\
& & & 73.6 & 100 & 63,328.9 & 1,340 & & & \\
& & & 2,122.8 & 1,000 & \cellcolor{yellow}OT & --- & & & \\\hline
\multirow{3}{*}{ nltcs } & \multirow{3}{*}{ 16 } & \multirow{3}{*}{ 3,236 } & 11.8 & 10 & 47,338.4 & 552 & \multirow{3}{*}{ 125.5 } & \multirow{3}{*}{ 652 } & \multirow{3}{*}{ 326 } \\
& & & 406.6 & 100 & \cellcolor{yellow}OT & --- & & & \\
& & & 13,224.6 & 1,000 & \cellcolor{yellow}OT & --- & & & \\\hline
msnbc & 17 & 58,265 & \cellcolor{yellow}ES & --- & \cellcolor{yellow}ES & --- & 4,018.9 & 24 & 24 \\\hline
\multirow{3}{*}{ letter } & \multirow{3}{*}{ 17 } & \multirow{3}{*}{ 20,000 } & 26.0 & 10 & 18,788.0 & 200 & \multirow{3}{*}{ 56,344.8 } & \multirow{3}{*}{ 20 } & \multirow{3}{*}{ 10 } \\
& & & 909.8 & 100 & \cellcolor{yellow}OT & --- & & & \\
& & & 41,503.9 & 1,000 & \cellcolor{yellow}OT & --- & & & \\\hline
\multirow{3}{*}{ voting } & \multirow{3}{*}{ 17 } & \multirow{3}{*}{ 435 } & 34.1 & 10 & 101.9 & 30 & \multirow{3}{*}{ 6.0 } & \multirow{3}{*}{ 621 } & \multirow{3}{*}{ 207 } \\
& & & 1,125.7 & 100 & 1,829.2 & 3,392 & & & \\
& & & 38,516.2 & 1,000 & 42,415.3 & 3,665 & & & \\\hline
\multirow{3}{*}{ zoo } & \multirow{3}{*}{ 17 } & \multirow{3}{*}{ 101 } & 33.5 & 10 & 99.8 & 52 & \multirow{3}{*}{ 8,418.8 } & \multirow{3}{*}{ 29,073 } & \multirow{3}{*}{ 6,761 } \\
& & & 1,041.7 & 100 & 1,843.4 & 100 & & & \\
& & & 41,412.1 & 1,000 & \cellcolor{yellow}OT & --- & & & \\\hline
\multirow{3}{*}{ hepatitis } & \multirow{3}{*}{ 20 } & \multirow{3}{*}{ 155 } & 351.2 & 10 & 872.3 & 89 & \multirow{3}{*}{ 441.4 } & \multirow{3}{*}{ 28,024 } & \multirow{3}{*}{ 3,534 } \\
& & & 13,560.3 & 100 & 20,244.7 & 842 & & & \\
& & & \cellcolor{yellow}OT & 1,000 & \cellcolor{yellow}OT & --- & & & \\\hline
\multirow{3}{*}{ parkinsons } & \multirow{3}{*}{ 23 } & \multirow{3}{*}{ 195 } & 3,908.2 & 10 & \cellcolor{yellow}OT & --- & \multirow{3}{*}{ 1,515.9 } & \multirow{3}{*}{ 150,000 } & \multirow{3}{*}{ 42,448 } \\
& & & \cellcolor{yellow}OT & 100 & \cellcolor{yellow}OT & --- & & & \\
& & & \cellcolor{yellow}OT & 1,000 & \cellcolor{yellow}OT & --- & & & \\\hline
autos & 26 & 159 & \cellcolor{yellow}OM & 1 & \cellcolor{yellow}OM & --- & \cellcolor{yellow}OT & --- & --- \\
insurance & 27 & 1,000 & \cellcolor{yellow}OM & 1 & \cellcolor{yellow}OM & --- & 8.3 & 1,081 & 133 
    \end{tabular}
    \caption{The search time $T$ and the number of collected networks $k$, $|\graphset_k|$ and $|\graphset_{20}|$ for KBest, KbestEC and GOBNILP\_dev (BF = 20) using BDeu, where $n$ is the number of random variables in the dataset, $N$ is the number of instances in the dataset, OM = Out of Memory, OT = Out of Time and ES = Error in Scoring. Note that $|\graphset_k|$ is the number of DAGs covered by the $k$-best MECs in KBestEC and $|\mathcal{M}_{20}|$ is the number of MECs in the networks collected by GOBNILP\_dev.}\label{TAB:kbest}
\end{table}

We first use GOBNILP\_dev to find the optimal score since GOBNILP\_dev takes objective limit ($\opt + \epsilon$) for enumerating feasible networks. For BIC, We set the limit on the size of the parent set based on Corollary~\ref{cor:decampossizerelaxed} that guarantees optimality, whereas for BDeu we set the number to $-1$ to allow all possible sizes. Then all networks falling into the limit are collected with a counting limit of 150,000. Finally the collected networks are categorized into Markov equivalence classes (MECs)\footnote{Our code can also collect only one DAG from each MEC.}, where two networks belong to the same MEC iff they have the same skeleton and v-structures~\citep{VermaP1990}. The proposed approach is tested on datasets with up to 57 variables. The search time $T$, the number of collected networks $|\graphset|$ and the number of MECs $\mathcal{M}$ in the collected networks at BF = 3, 20 and 150 using BIC are reported in Table~\ref{TAB:bf}, where $n$ is the number of random variables in the dataset and $N$ is the number of instances in the dataset. The three thresholds are chosen according to the interpreting scale suggested by~\citet{HeckermanGC95} where 3 marks the transition between anecdotal and positive evidence, 20 marks the transition between positive and strong evidence and 150 marks the transition between strong and very strong evidence. The search time mostly depends on a combined effect of the size of the network, the sample size and the number of MECs at a given BF. Some fairly large networks such as alarm, sponge and barley are solved much faster than smaller networks with a large sample size; e.g., msnbc and letter.

The results also indicate that the number of collected networks and the number of MECs at three BF levels varies substantially across different datasets. In general, datasets with smaller sample sizes tend to have more networks collected at a given BF since near-optimal networks have similar posterior probabilities to the best network. Although the desired level of BF for a study, like the p-value, is often determined with domain knowledge, the proposed approach, given sufficient samples, will produce meaningful results that can be used for further analysis.

\subsection{Synthetic Data}\label{sec:synth_data}

We use BNs up to 76 nodes from the Bayesian Network Repository to generate synthetic data in the form of 10 random samples for various sample sizes up to 1,000. Near optimal networks are collected following the same procedure outlined in Section~\ref{sec:bf_approach} using BIC. For the three largest datasets, hailfinder, hepar2, and win95pts, the scoring process failed to complete within 24 hours for the sample size of 1,000.

The average number of networks in the credible sets is reported in Table~\ref{tab:syn_num_net}. The results are consistent with Table~\ref{TAB:bf} in demonstrating that the number of collected networks are specific to the dataset. Collecting a large number of networks is not always ideal for large datasets, e.g., water, mildew, and barley all have a very small number of networks in the credible sets comparing to some other datasets with fewer nodes. We also note that increasing sample size does not always lead to smaller credible sets. Instead, the number of networks in the credible sets tend to peak around certain sample sizes. For example, the largest credible sets for water are collected with a sample size of 500, although the numbers are quite different sometimes across 10 trials as indicated by the large standard deviation.

\begin{table}[!ht]
\tiny
\centering
\begin{tabularx}{\textwidth}{ll||l*{4}{X}|*{4}{X}|*{4}{X}}
\multicolumn{1}{c}{\multirow{2}{*}{Data}} & \multicolumn{1}{c||}{\multirow{2}{*}{n}} & BF   & \multicolumn{4}{c|}{3}                                                     & \multicolumn{4}{c|}{20}                                                     & \multicolumn{4}{c}{150}                                                      \\
\multicolumn{1}{c}{}                      & \multicolumn{1}{c||}{}                   & S.S. & \multicolumn{1}{c}{50}               & \multicolumn{1}{c}{100}               & \multicolumn{1}{c}{500}               & \multicolumn{1}{c|}{1000}            & \multicolumn{1}{c}{50}               & \multicolumn{1}{c}{100}               & \multicolumn{1}{c}{500}               & \multicolumn{1}{c|}{1000}            & \multicolumn{1}{c}{50}               & \multicolumn{1}{c}{100}               & \multicolumn{1}{c}{500}               & \multicolumn{1}{c}{1000}              \\ \hline
cancer     & 5  &      & 25 $\pm$ 19       & 12 $\pm$ 13       & 6 $\pm$ 4         & 6 $\pm$ 3       & 320 $\pm$ 256     & 117 $\pm$ 84      & 46 $\pm$ 23       & 23 $\pm$ 14     & 1286 $\pm$ 779    & 559 $\pm$ 333     & 160 $\pm$ 78      & 96 $\pm$ 67       \\
earthquake & 5  &      & 18 $\pm$ 19       & 20 $\pm$ 26       & 5 $\pm$ 3         & 2 $\pm$ 3       & 149 $\pm$ 125     & 92 $\pm$ 85       & 16 $\pm$ 7        & 6 $\pm$ 7       & 558 $\pm$ 332     & 356 $\pm$ 202     & 45 $\pm$ 23       & 20 $\pm$ 10       \\
survey     & 6  &      & 11 $\pm$ 10       & 7 $\pm$ 3         & 7 $\pm$ 5         & 9 $\pm$ 7       & 139 $\pm$ 129     & 69 $\pm$ 42       & 41 $\pm$ 27       & 32 $\pm$ 20     & 841 $\pm$ 686     & 385 $\pm$ 255     & 164 $\pm$ 85      & 115 $\pm$ 47      \\
asia       & 8  &      & 87 $\pm$ 44       & 169 $\pm$ 299     & 7 $\pm$ 5         & 5 $\pm$ 5       & 1872 $\pm$ 901    & 1928 $\pm$ 1847   & 50 $\pm$ 18       & 28 $\pm$ 32     & 21228 $\pm$ 11040 & 16523 $\pm$ 12958 & 208 $\pm$ 75      & 126 $\pm$ 154     \\
sachs      & 11 &      & 65 $\pm$ 32       & 56 $\pm$ 43       & 45 $\pm$ 46       & 209 $\pm$ 98    & 491 $\pm$ 654     & 194 $\pm$ 179     & 84 $\pm$ 101      & 267 $\pm$ 148   & 3153 $\pm$ 5098   & 694 $\pm$ 592     & 197 $\pm$ 230     & 334 $\pm$ 162     \\
child      & 20 &      & 1659 $\pm$ 1165   & 330 $\pm$ 349     & 57 $\pm$ 36       & 44 $\pm$ 38     & 62298 $\pm$ 35163 & 9789 $\pm$ 13872  & 223 $\pm$ 191     & 122 $\pm$ 98    & 100000 $\pm$ 0    & 55106 $\pm$ 39659 & 744 $\pm$ 519     & 316 $\pm$ 241     \\
insurance  & 27 &      & 6286 $\pm$ 5935   & 1414 $\pm$ 2656   & 54 $\pm$ 25       & 11 $\pm$ 9      & 82626 $\pm$ 36908 & 23096 $\pm$ 30691 & 251 $\pm$ 199     & 42 $\pm$ 26     & 96616 $\pm$ 10701 & 71660 $\pm$ 39118 & 1487 $\pm$ 1190   & 198 $\pm$ 154     \\
water      & 32 &      & 499 $\pm$ 361     & 3326 $\pm$ 3115   & 10381 $\pm$ 12183 & 2656 $\pm$ 2284 & 4135 $\pm$ 4087   & 10960 $\pm$ 12946 & 22788 $\pm$ 25099 & 8611 $\pm$ 8289 & 23443 $\pm$ 29178 & 29835 $\pm$ 29773 & 40285 $\pm$ 33522 & 16430 $\pm$ 13155 \\
mildew     & 35 &      & 13 $\pm$ 8        & 47 $\pm$ 32       & 564 $\pm$ 307     & 1139 $\pm$ 280  & 30 $\pm$ 33       & 71 $\pm$ 55       & 664 $\pm$ 330     & 1625 $\pm$ 1010 & 65 $\pm$ 43       & 136 $\pm$ 108     & 1040 $\pm$ 592    & 2344 $\pm$ 2204   \\
alarm      & 37 &      & 79174 $\pm$ 36149 & 23361 $\pm$ 40727 & 369 $\pm$ 463     & 131 $\pm$ 109   & 100000 $\pm$ 0    & 78512 $\pm$ 35492 & 10039 $\pm$ 11613 & 2186 $\pm$ 1755 & 100000 $\pm$ 0    & 100000 $\pm$ 0    & 60723 $\pm$ 42093 & 29146 $\pm$ 26129 \\
barley     & 48 &      & 4 $\pm$ 2         & 26 $\pm$ 5        & 4669 $\pm$ 3371   & 331 $\pm$ 129   & 5 $\pm$ 2         & 36 $\pm$ 26       & 10247 $\pm$ 4978  & 502 $\pm$ 245   & 23 $\pm$ 9        & 138 $\pm$ 19      & 21090 $\pm$ 11262 & 1014 $\pm$ 575    \\
hailfinder & 56 &      & 99046 $\pm$ 3018  & 100000 $\pm$ 0    & 100000 $\pm$ 0    & \yc OT      & 100000 $\pm$ 0    & 100000 $\pm$ 0    & 100000 $\pm$ 0    & \yc OT      & 100000 $\pm$ 0    & 100000 $\pm$ 0    & 100000 $\pm$ 0    & \yc OT        \\
hepar2     & 70 &      & 100000 $\pm$ 0    & 100000 $\pm$ 0    & 82994 $\pm$ 32242 & \yc OT      & 100000 $\pm$ 0    & 100000 $\pm$ 0    & 100000 $\pm$ 0    & \yc OT      & 100000 $\pm$ 0    & 100000 $\pm$ 0    & 100000 $\pm$ 0    & \yc OT        \\
win95pts   & 76 &      & 100000 $\pm$ 0    & 100000 $\pm$ 0    & 100000 $\pm$ 0    & \yc OT      & 100000 $\pm$ 0    & 100000 $\pm$ 0    & 100000 $\pm$ 0    & \yc OT      & \yc OT        & \yc OT        & 100000 $\pm$ 0    & \yc OT       
      
\end{tabularx}
\caption{The average number of networks $\pm$ standard deviation in the credible sets with various Bayes factors (BFs) and sample sizes (S.S.)}\label{tab:syn_num_net}
\end{table}

The average number of equivalence classes is reported in Table~\ref{tab:syn_num_ec}. The large number of networks can indeed be represented by a handful of equivalence classes. Increasing the amount of training data can both lead to a decrease in the number of networks and an increase in the complexity of the collected networks. Although the former is more evident for most datasets, water, mildew, and barley are examples of the latter. There are still peaks of sample sizes where large numbers of equivalence classes are collected, but the scenario occurs much less frequently than in Table~\ref{tab:syn_num_net}.

\begin{table}[!ht]
\tiny
\centering
\begin{tabularx}{\textwidth}{ll||l*{4}{X}|*{4}{X}|*{4}{X}}
\multicolumn{1}{c}{\multirow{2}{*}{Data}} & \multicolumn{1}{c||}{\multirow{2}{*}{n}} & BF   & \multicolumn{4}{c|}{3}                                                     & \multicolumn{4}{c|}{20}                                                     & \multicolumn{4}{c}{150}                                                      \\
\multicolumn{1}{c}{}                      & \multicolumn{1}{c||}{}                   & S.S. & \multicolumn{1}{c}{50}               & \multicolumn{1}{c}{100}               & \multicolumn{1}{c}{500}               & \multicolumn{1}{c|}{1000}            & \multicolumn{1}{c}{50}               & \multicolumn{1}{c}{100}               & \multicolumn{1}{c}{500}               & \multicolumn{1}{c|}{1000}            & \multicolumn{1}{c}{50}               & \multicolumn{1}{c}{100}               & \multicolumn{1}{c}{500}               & \multicolumn{1}{c}{1000}              \\ \hline
cancer     & 5  &      & 10 $\pm$ 8        & 4 $\pm$ 3         & 2 $\pm$ 2         & 2 $\pm$ 1  & 94 $\pm$ 72       & 36 $\pm$ 23       & 14 $\pm$ 7        & 7 $\pm$ 4   & 368 $\pm$ 224     & 159 $\pm$ 92      & 47 $\pm$ 21       & 28 $\pm$ 18    \\
earthquake & 5  &      & 5 $\pm$ 5         & 5 $\pm$ 5         & 2 $\pm$ 1         & 1 $\pm$ 1  & 37 $\pm$ 25       & 22 $\pm$ 17       & 4 $\pm$ 1         & 2 $\pm$ 1   & 145 $\pm$ 78      & 89 $\pm$ 43       & 10 $\pm$ 5        & 5 $\pm$ 2      \\
survey     & 6  &      & 5 $\pm$ 3         & 4 $\pm$ 2         & 2 $\pm$ 1         & 2 $\pm$ 1  & 40 $\pm$ 29       & 24 $\pm$ 13       & 10 $\pm$ 5        & 6 $\pm$ 3   & 213 $\pm$ 141     & 111 $\pm$ 65      & 37 $\pm$ 19       & 21 $\pm$ 9     \\
asia       & 8  &      & 15 $\pm$ 6        & 27 $\pm$ 41       & 2 $\pm$ 1         & 2 $\pm$ 1  & 283 $\pm$ 123     & 314 $\pm$ 289     & 13 $\pm$ 4        & 7 $\pm$ 6   & 3250 $\pm$ 1597   & 2669 $\pm$ 2056   & 58 $\pm$ 19       & 28 $\pm$ 28    \\
sachs      & 11 &      & 5 $\pm$ 3         & 3 $\pm$ 3         & 2 $\pm$ 1         & 2 $\pm$ 1  & 37 $\pm$ 45       & 12 $\pm$ 12       & 3 $\pm$ 1         & 2 $\pm$ 1   & 226 $\pm$ 333     & 45 $\pm$ 42       & 6 $\pm$ 3         & 3 $\pm$ 1      \\
child      & 20 &      & 51 $\pm$ 40       & 13 $\pm$ 13       & 4 $\pm$ 2         & 3 $\pm$ 2  & 2808 $\pm$ 2477   & 415 $\pm$ 547     & 14 $\pm$ 11       & 7 $\pm$ 5   & 16097 $\pm$ 6732  & 5045 $\pm$ 6483   & 46 $\pm$ 31       & 19 $\pm$ 13    \\
insurance  & 27 &      & 128 $\pm$ 146     & 27 $\pm$ 28       & 5 $\pm$ 2         & 2 $\pm$ 2  & 5107 $\pm$ 5182   & 660 $\pm$ 761     & 23 $\pm$ 19       & 9 $\pm$ 5   & 15115 $\pm$ 10086 & 4522 $\pm$ 3334   & 137 $\pm$ 119     & 40 $\pm$ 32    \\
water      & 32 &      & 6 $\pm$ 4         & 2 $\pm$ 1         & 4 $\pm$ 6         & 3 $\pm$ 2  & 47 $\pm$ 49       & 7 $\pm$ 4         & 11 $\pm$ 14       & 8 $\pm$ 6   & 226 $\pm$ 274     & 23 $\pm$ 14       & 22 $\pm$ 24       & 17 $\pm$ 12    \\
mildew     & 35 &      & 2 $\pm$ 1         & 1 $\pm$ 1         & 1 $\pm$ 0         & 1 $\pm$ 0  & 3 $\pm$ 3         & 2 $\pm$ 1         & 2 $\pm$ 1         & 2 $\pm$ 1   & 7 $\pm$ 4         & 4 $\pm$ 2         & 2 $\pm$ 1         & 2 $\pm$ 3      \\
alarm      & 37 &      & 2477 $\pm$ 2580   & 287 $\pm$ 457     & 14 $\pm$ 9        & 5 $\pm$ 5  & 10608 $\pm$ 7214  & 5428 $\pm$ 4312   & 440 $\pm$ 448     & 80 $\pm$ 64 & 13960 $\pm$ 8921  & 34705 $\pm$ 16156 & 6744 $\pm$ 7224   & 1113 $\pm$ 868 \\
barley     & 48 &      & 1 $\pm$ 0         & 1 $\pm$ 0         & 2 $\pm$ 1         & 2 $\pm$ 0  & 2 $\pm$ 1         & 2 $\pm$ 1         & 6 $\pm$ 2         & 3 $\pm$ 1   & 5 $\pm$ 1         & 4 $\pm$ 1         & 11 $\pm$ 4        & 5 $\pm$ 3      \\
hailfinder & 56 &      & 299 $\pm$ 437     & 55 $\pm$ 41       & 29730 $\pm$ 28269 & \yc OT & 445 $\pm$ 445     & 1250 $\pm$ 2758   & 30261 $\pm$ 28182 & \yc OT  & 3546 $\pm$ 5619   & 587 $\pm$ 306     & 37203 $\pm$ 35406 & \yc OT     \\
hepar2     & 70 &      & 29056 $\pm$ 17518 & 24427 $\pm$ 15628 & 8129 $\pm$ 6811   & \yc OT & 47869 $\pm$ 28169 & 32408 $\pm$ 21816 & 23515 $\pm$ 22694 & \yc OT  & 39769 $\pm$ 26715 & 38219 $\pm$ 27572 & 21942 $\pm$ 28073 & \yc OT     \\
win95pts   & 76 &      & 44666 $\pm$ 32398 & 21751 $\pm$ 19698 & 21679 $\pm$ 21324 & \yc OT & 34578 $\pm$ 19408 & 30477 $\pm$ 34155 & 23359 $\pm$ 20296 & \yc OT  & \yc OT        & \yc OT        & 27264 $\pm$ 30802 & \yc OT    
      
\end{tabularx}
\caption{The average number of equivalence classes $\pm$ standard deviation in the credible sets with various Bayes factors (BFs) and sample sizes (S.S.)}\label{tab:syn_num_ec}
\end{table}

\subsection{Bayes Factor vs. K-Best}\label{sec:bf_v_k}

In this section, we compare our approach with published solvers that are able to find a subset of top-scoring networks with the given parameter $k$. The solvers under consideration are KBest\_12b\footnote{\url{http://web.cs.iastate.edu/~jtian/Software/UAI-10/KBest.htm}} from~\citep{TianHR10}, KBestEC\footnote{\url{http://web.cs.iastate.edu/~jtian/Software/AAAI-14-yetian/KBestEC.htm}} from~\citep{ChenT2014}, and GOBNILP 1.6.3~\citep{BartlettC13},
referred to as KBest, KBestEC and GOBNILP below. The first two solvers are based on the dynamic programming approach introduced in~\citep{SilanderM06}. 
Due to the lack of support for BIC in KBest and KBestEC, only BDeu with a equivalent sample size of one is used in corresponding experiments.

The most recent stable version of \gobnilp{} is 1.6.3 that works with SCIP 3.2.1. The default configuration is used and experiments are conducted for both BIC and BDeu scoring functions. However, the $k$-best results are omitted here due to its poor performance. Despite that \gobnilp{} can iteratively find the $k$-best networks in descending order by adding linear constraints, the pruning rules designed to find the best network are turned off to preserve sub-optimal networks. In fact, the memory usage often exceeded 64 GB during the initial ILP formulation, indicating that the lack of pruning rules posed serious challenge for GOBNILP. \gobnilp{\_dev}, on the other hand, can take advantage of the pruning rules presented above in the proposed BF approach and its results compare favorably to KBest and KBestEC.

The experimental results of KBest, KBestEC and \gobnilp{\_dev} are reported in Table~\ref{TAB:kbest}, where $n$ is the number of random variables in the dataset, $N$ is the number of instances in the dataset, and $k$ is the number of top scoring networks. The search time $T$ is reported for KBest, KBestEC and GOBNILP\_dev (BF = 20). The number of DAGs covered by the $k$ MECs $|\graphset_k|$ is reported for KBestEC. In comparison, the last two columns are the number of found networks $|\graphset_{20}|$ and the number of MECs $|\mathcal{M}_{20}|$ using the BF approach with a given BF of 20 and BDeu scoring function.

As the number of requested networks $k$ increases, the search time for both KBest and KBestEC grows exponentially. The KBest and KBestEC are designed to solve problems of size fewer than 20, and so they have some difficulty with larger datasets.\footnote{Obtained through correspondence with the author.} They also fail to generate correct scoring files for msnbc. KBestEC seems to successfully expand the coverage of DAGs with some overhead for checking equivalence classes. However, KBestEC took much longer than KBest for some instances, e.g., nltcs and letter, and the number of DAGs covered by the found MECs is inconsistent for nltcs, letter and zoo. The search time for the BF approach is improved over the $k$-best approach except for datasets with very large sample sizes. The generalized pruning rules are very effective in reducing the search space, which then allows GOBNILP\_dev to solve the ILP problem subsequently. Comparing to the improved results in~\citep{ChenCD2015,ChenCD2016}, our approach can scale to larger networks if the scoring file can be generated.\footnote{We are unable to generate BDeu score files for datasets with 30 or more variables.}

\begin{figure}[thb]
    \centering
    \includegraphics[width=0.7\linewidth]{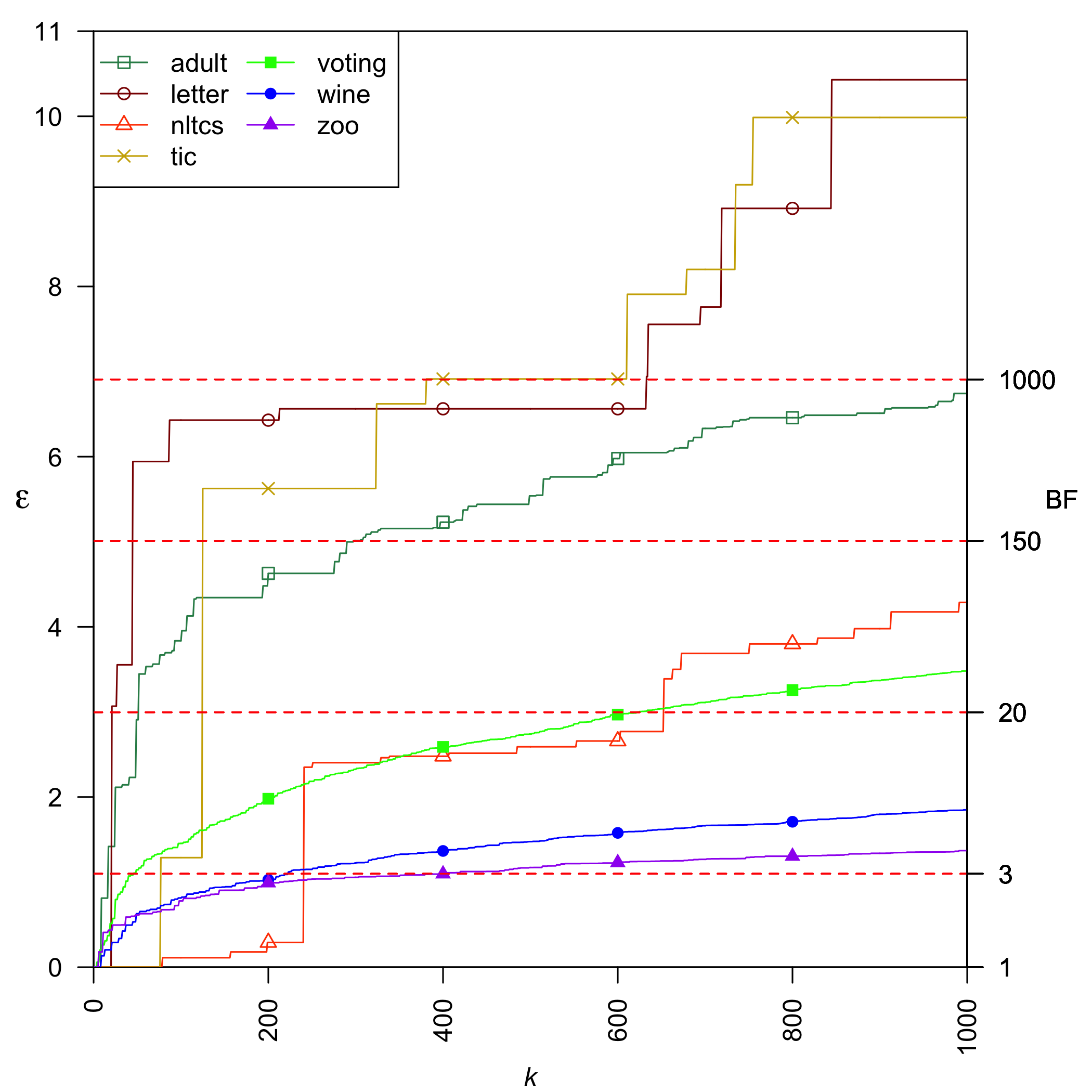}
    \caption{The deviation $\epsilon$ from the optimal BDeu score by $k$ using results from KBest. The corresponding values of the BF ($\epsilon=\log(BF)$, see Equation \ref{EQUATION:bf}) are presented on the right. For example, if the desired BF value is 20, then all networks falling below the dash line at 20 are credible.}\label{fig:diff}
\end{figure}

Now we show that different datasets have distinct score patterns in the top scoring networks. The scores of the 1,000-best networks for some datasets in the KBest experiment are plotted in Figure~\ref{fig:diff}. A specific line for a dataset indicates the deviation $\epsilon$ from the optimal BDeu score by the $k$th-best network. For reference, the red dash lines represent different levels of BFs calculated by $\epsilon=\log{BF}$ (See Equation~\ref{EQUATION:bf}). The figure shows that it is difficult to pick a value for $k$ \emph{a priori} to capture the appropriate set of top scoring networks. For a few datasets such as adult and letter, it only takes fewer than 50 networks to reach a BF of 20, whereas zoo needs more than 10,000 networks. The sample size has a significant effect on the number of networks at a given BF since the lack of data leads to many BNs with similar probabilities. It would be reasonable to choose a large value for $k$ in model averaging when data is scarce and vice versa, but only the BF approach is able to automatically find the appropriate and credible set of networks for further analysis.

\section{Conclusion}\label{sec:conclusion}

Existing approaches for model averaging for Bayesian network structure learning
either severely restrict the structure of the Bayesian network
or have only been shown to scale to networks with fewer than 30 random
variables. In this paper, we proposed a novel approach to model
averaging in Bayesian network structure learning that 
finds all networks within a factor of optimal. 
Our approach has two primary advantages. First,
our approach only considers \emph{credible} models in that they
are optimal or near-optimal in score. Second, our approach
is significantly more efficient and scales to much larger
Bayesian networks than existing approaches.
We modified GOBNILP~\citep{BartlettC13}, a state-of-the-art method for finding
an optimal Bayesian network, to implement our generalized pruning rules
and to find all \emph{near}-optimal networks. Our experimental results demonstrate
that the modified GOBNILP scales to significantly
larger networks without resorting to restricting the structure 
of the Bayesian networks that are learned.

\clearpage
\newpage
\appendix
\section{Proofs of Pruning Rules}

We discuss the original pruning rules and prove their generalization below. A candidate parent set $\Pi_i$ can be \textit{safely pruned} given a non-negative constant $\epsilon \in \mathbb{R}^+$ if $\Pi_i$ cannot be the parent set of $V_i$ in any network in the set of credible networks. Note that proofs of the original rules can be obtained by setting $\epsilon=0$.

\subsection{Proof of Lemma~\ref{lem:scoreprune}}

\citet{TeyssierK05} give a pruning rule that is applicable for all decomposable scoring functions.
\begin{theorem}\citep{TeyssierK05}
	Given a vertex variable $\vertex{j}$, and candidate parent sets
	$\Pi_j$ and  $\Pi_j^{\prime}$, if $\Pi_j \subset \Pi_j^{\prime}$ and $\score{}{\Pi_j} < \score{}{\Pi_j^{\prime}}$,
	$\Pi_j^{\prime}$ can be safely pruned. \label{thm:basicrule}
\end{theorem}


We relax this pruning rule and prove Lemma~\ref{lem:scoreprune} below.

	\begin{proof}(Lemma~\ref{lem:scoreprune})
Consider networks $\graph$ and $\graph'$ that are the same except for the parent set of $\vertex{j}$, where $\graph$ has the parent set $\Pi_j$ for $\vertex{j}$ and $\graph'$ has the parent set $\Pi_j^{\prime}$ for $\vertex{j}$. 
\begin{align*}
    \score{}{\graph} &=  \score{}{\Pi_j} + \sum_{i\neq j} \score{}{\Pi_j}  &[\score{ }{}\text{ is decomposable}]\\
                    &\leq\score{}{\Pi_j^{\prime}} + \epsilon+ \sum_{i\neq j} \score{}{\Pi_j}  &[\text{given}]\\
                    &=\score{}{\graph'}.
\end{align*}
Thus, $\graph'$ cannot be in the set of credible  networks.
	\end{proof}
	
\subsection{Proof of Theorem \ref{thm:decampospar}}

An additional pruning rule can be derived from Theorem \ref{thm:basicrule} that is applicable to the BIC/MDL scoring function. 

\begin{theorem}\citep{CamposJ11}
Given a vertex variable $\vertex{i}$, and candidate parent sets $\Pi_i$ and $\Pi_i^{\prime}$,
	if $\Pi_i \subset \Pi_i^{\prime}$ and $\score{}{\Pi_i} -  \pen{}{\Pi_i^{\prime}} < 0$,
	$\Pi_i^{\prime}$ and all supersets of $\Pi_i^{\prime}$ can be safely pruned  if $\sigma$ is the BIC/MDL scoring function.    
\end{theorem}


Here, $\pen{}{\Pi_i^{\prime}}$ is the penalty term in the BIC scoring function. This pruning rule is relaxed as Theorem~\ref{thm:decampospar} and we prove it below. 

\begin{proof}(Theorem~\ref{thm:decampospar})
    \begin{align*}
       &\score{}{\Pi_i} -  \pen{}{\Pi_i^{\prime}}  + \epsilon < 0  &[\text{given}]\\
       &\Rightarrow -\score{}{\Pi_i} + \pen{}{\Pi_i^{\prime}} -\epsilon > 0\\
       &\Rightarrow -\score{}{\Pi_i} + \pen{}{\Pi_i^{\prime}} - L(\Pi_i^{\prime})  - \epsilon > 0 &[L(\Pi_i^{\prime}) < 0]\\
       &\Rightarrow \score{}{\Pi_i^{\prime}} > \score{}{\Pi_i} + \epsilon.
    \end{align*}
    By Lemma \ref{lem:scoreprune}, $\Pi_i^{\prime}$ cannot be an optimal parent set. Using the fact that penalties increase with increase in parent set size, supersets of $\Pi_i^{\prime}$ cannot be in the set of credible  networks. The result follows.
\end{proof}

\subsection{Proof of Theorem~\ref{THEOREM:decamposbicrelaxed}}


\begin{theorem}\citep{CamposJ11}
	 Given a vertex variable $V_i$ and candidate parent set $\Pi_i$ such that $r_{\Pi_i}> \frac{N}{w} \frac{\log r_i}{r_i -1}$,  if $\parents_i \subsetneq \parents_i'$ , then $\parents_i'$ can be safely pruned if $\sigma$ is the BIC scoring function. \label{THEOREM:decamposbic}
\end{theorem}

We relax the pruning rule given in Theorem~\ref{THEOREM:decamposbic} as Theorem~\ref{THEOREM:decamposbicrelaxed} and prove it below.

\begin{proof}(Theorem~\ref{THEOREM:decamposbicrelaxed})
\begin{align*}
    &\sigma(\Pi_i') - \sigma(\Pi_i)\\
    & \stackbin[]{0}{=} -\displaystyle \max_{\theta_i} L(\Pi_i') + t(\Pi_i')\cdot w + \displaystyle \max_{\theta_i} L(\Pi_i) - t(\Pi_i)\cdot w \\
    & \stackbin[]{1}{\geq} \displaystyle -\max_{\theta_i} L(\Pi_i) + t(\Pi_i')\cdot w  - t(\Pi_i)\cdot w \\
    &\stackbin[]{2}{=} \displaystyle -\sum_{j=1}^{r_{\Pi_i}}n_{ij}(-\sum^{r_i}_{i=1}\frac{n_{ijk}}{n_{ij}}\log \frac{n_{ijk}}{n_{ij}} )+ t(\Pi_i')\cdot w - t(\Pi_i)\cdot w \\
    & \stackbin[]{3}{\geq} \displaystyle -\sum_{j=1}^{r_{\Pi_i}}n_{ij}H(\theta_{ij})- t(\Pi_i')\cdot w  + t(\Pi_i)\cdot w \\
    & \stackbin[]{4}{\geq}  \displaystyle -\sum_{j=1}^{r_{\Pi_i}}n_{ij}\log r_i + r_{\Pi_i}\cdot(r_e -1)\cdot (r_i -1)\cdot w \\
    & \stackbin[]{5}{\geq}  \displaystyle -\sum_{j=1}^{r_{\Pi_i}}n_{ij}\log r_i + r_{\Pi_i}\cdot (r_i -1)\cdot w \\ 
    & \stackbin[]{6}{=}  -N\log r_i + r_{\Pi_i}\cdot (r_i -1)\cdot w\\
    & \stackbin[]{7}{>} \epsilon.
\end{align*}
    Step 0  uses the definition of $BIC$. Step 1 uses $ \max_{\theta_i} L(\Pi_i')$ is negative. Step 2 uses the fact that the maximum likelihood estimate, $\theta^\ast_{ijk} = \frac{n_{ijk}}{n_{ij}}$ and $n_{ij} = \sum_{i=1}^{r_{i}}n_{ijk}$. Step 3 uses the definition of entropy. Step 4 uses the definition of the penalty function $t$. Step 5 uses $r_e \geq 2$.  Finally, the RHS in Step 6 follows because of the definition of $n_{ij}$. Step 7 uses the assumption of the theorem. 
    
    Using Lemma \ref{lem:scoreprune}, we get the result as desired.  
	\end{proof}
	
\subsection{Proof of Corollary~\ref{cor:decampossizerelaxed}}
\begin{corollary}\citep{CamposJ11}
	Given a vertex variable $\vertex{i}$ and candidate parent set 
	$\Pi_i$, if $\Pi_i$ has more than $\log_2 N$ elements,
	$\Pi_i$ can be safely pruned if $\sigma$ is the BIC scoring function. \label{cor:decampossize}
\end{corollary}

Using Theorem \ref{THEOREM:decamposbicrelaxed}, we generalize Corollary~\ref{cor:decampossize} to Corollary~\ref{cor:decampossizerelaxed} and prove it below.

	\begin{proof}(Corollary~\ref{cor:decampossizerelaxed})
	Assuming $N>4$, take a variable $V_i$ and a parent set $\Pi_i$
with $|\Pi_i| = \ceil{\log_2 (N + \epsilon)}$ elements.  Because every variable has at least two states, we know that $r_{\Pi_i} \geq 2^{|\Pi_i|}  \geq N +\epsilon > \frac{N}{w}\frac{\log r_i}{r_i -1}+ \epsilon$, because $w =\log \frac{N}{2}$ gives us $\frac{\log r_i}{w(r_i -1)} < 1$ , and by Theorem \ref{THEOREM:decamposbicrelaxed} we know that no proper superset of $\Pi_i$ can be an optimal parent set for $\vertex{i}$ as desired.
	\end{proof} 

\subsection{Proof of Theorem~\ref{thm:entropyrelaxed}}

\begin{lemma}\citep{Campos2017}
Given a vertex variable $V_i$, and candidate parent sets $\Pi_i$, $\Pi_i'$ such that $\Pi_i = \Pi_i' \cup \{V_j\}$ for some variable $V_j \notin \Pi_i'$, we have  $L(\Pi_i) - L(\Pi_i') \leq  N \cdot  \min\{H(V_i \mid \Pi_i'), H(V_j  \mid \Pi_i')\}$. \label{lem:entropy1}
\end{lemma}

\begin{proof}
First, consider the definition of $L_{i}(\Pi_i)$,
\begin{equation*}
    L(\Pi_i) = \max_{\theta} \sum_{j=1}^{r_{\Pi_i}}\sum_{k=1}^{r_i} n_{ijk} \log {\theta}_{ijk} ,
\end{equation*}
where the maximum likelihood estimate of $\theta_{ijk}$ is $\frac{n_{ijk}}{n_{ij}}$.
This gives us $N\cdot H(V_i \mid \Pi_i) = -L(\Pi_i)$. Thus, we get,
\begin{align*}
    L(\Pi_i) - L(\Pi_i') &= N\cdot (H(V_i \mid  \Pi_i') - H(V_i \mid  \Pi_i)\stackbin[]{1}{\leq} N\cdot H(V_i \mid  \Pi_i') .
\end{align*}
We use the fact that entropy is positive. Now, consider the definition of mutual information,
\begin{align*}
    I(X,Y \mid Z) = H(X \mid Z) = H(X \mid Y \cup Z).
\end{align*}
This gives us,
\begin{align*}
    L(\Pi_i) - L(\Pi_i') &= N \cdot I(V_i, V_j  \mid  \Pi_i')\\
                        & \stackbin[]{2}{=} N\cdot (H(V_j \mid  \Pi_i') - H(V_j \mid \Pi_i' \cup \{V_i\}))\\
    \Rightarrow L(\Pi_i) - L(\Pi_i') &\stackbin[]{3}{\leq} N \cdot  \min\{H(V_i \mid \Pi_i'), H(V_j  \mid \Pi_i')\}.
\end{align*}
Step 3 combines Steps 1 and 2. The result follows as desired.
\end{proof}

\begin{theorem}\citep{Campos2017}
Given a vertex variable $V_i$, and candidate parent set $\Pi_i$, let $V_j \notin \Pi_i$  such that $N \cdot \min \{H(V_i \mid  \Pi_i), H(V_j  \mid \Pi_i)\} \geq (1 - r_{j}) \cdot t(\Pi_i)$. Then the candidate parent set $\Pi_i' = \Pi_i \cup \{V_j \}$ and all its supersets can be safely pruned if $\sigma$ is the BIC scoring function. 
\label{thm:entropy}
\end{theorem}

We relax Theorem~\ref{thm:entropy} and prove its generalization below.

\begin{proof}(Theorem~\ref{thm:entropyrelaxed})
\begin{align*}
    \score{}{\Pi_i'} 
                    &\stackbin[]{0}{=} -L(\Pi_i')  +  t(\Pi_i')\\
                    &\stackbin[]{1}{\geq} -L(\Pi_i) - N \cdot \min\{H(V_i \mid \Pi_i); H(V_j \mid \Pi_i)\} + t(\Pi_i')\\
                    &\stackbin[]{2}{\geq} -L(\Pi_i) + (1 - r_j)\cdot t(\Pi_i) +\epsilon  +t(\Pi_i') \\
                    &\stackbin[]{3}{=} -L(\Pi_i) + t(\Pi_i) - r_j \cdot t(\Pi_i) + \epsilon + t(\Pi_i') \\
                    &\stackbin[]{4}{=} -L(\Pi_i) + t(\Pi_i) - r_j \cdot r_{\Pi_i} \cdot (r_i -1) + \epsilon + t(\Pi_i') \\
                    &\stackbin[]{5}{=} -L(\Pi_i) + t(\Pi_i) - t(\Pi_i') + \epsilon + t(\Pi_i') \\
                    &\stackbin[]{6}{=} \score{}{\Pi_i} + \epsilon.
\end{align*}
Step 1 uses Lemma \ref{lem:entropy1}. Step 2 uses the assumptions of the question. Step 4 uses the definition of $t$. Step 5 uses $\Pi_i' = \Pi_i \cup \{V_j \}$. Using Lemma \ref{lem:scoreprune}, the result follows as desired.
\end{proof}

\subsection{Proof of Theorem~\ref{thm:jamesrelaxed}}

\begin{lemma}
Let $n_{ij}$ be a positive integer and $\alpha'$ be a positive real number. Then
\begin{equation*}
    \log \frac{\Gamma(n_{ij} + \alpha')}{\Gamma(\alpha')} = \sum_{i=0}^{n_{ij -1}}\log (i + \alpha')
\end{equation*} \label{lem:gamma}
\end{lemma}

\begin{proof}
We start with the property that $\Gamma(x+1) = x\Gamma(x)$ for any positive real number $x$. As $\alpha' > 0$, this gives us,
\begin{align*}
    \frac{\Gamma(1 + \alpha')}{\Gamma(\alpha') } &\stackbin[]{0}{=} \alpha' \\
     \frac{\Gamma(2 + \alpha')}{ \Gamma(1+ \alpha')} &\stackbin[]{1}{=} (1+ \alpha') \\
    \Rightarrow    \frac{\Gamma(1 + \alpha')\cdot\Gamma(2 + \alpha')}{ \Gamma(1+ \alpha')\Gamma(\alpha') }  &\stackbin[]{2}{=} \alpha' (1+ \alpha')\\
    \Rightarrow  \frac{\Gamma(1 + \alpha')\cdots\Gamma(n_{ij
    }+ \alpha')}{ \Gamma(n_{ij} -1 + \alpha')\cdots\Gamma(\alpha')} &\stackbin[]{3}{=} \alpha'\cdots (n_{ij} -1 + \alpha') \\
    \Rightarrow  \frac{\Gamma(n_{ij} + \alpha') }{ \Gamma(\alpha')} &\stackbin[]{4}{=} \alpha'\cdots (n_{ij} -1 + \alpha')\\
    \Rightarrow \log \frac{\Gamma(n_{ij}+ \alpha')}{\Gamma(\alpha')} &\stackbin[]{5}{=} \sum_{i=0}^{n_{ij -1}}\log (i + \alpha').
\end{align*}
Step 1 uses $1 + \alpha'$. Step 2 follows by multiplication of the equations in Step 1 and Step 0. Step 3 follows by repeated application of the identity. Step 4 cancels identical terms in the LHS. The result follows as desired.
\end{proof}

\begin{lemma}{6B}
Let $\{n_{ijk}\}_{k=1,...r_{i}}$ be non-negative integers with a positive sum, $n_{ij} = \sum_{k=1}^{r_i} n_{ijk}$ and $\alpha''$ be a positive real number. Then
\begin{equation*}
    \sum_{k=1}^{r_i} \log \frac{\Gamma(n_{ijk} + \alpha'')}{\Gamma(\alpha'')} \leq \log \frac{\Gamma(n_{ij} + \alpha'')}{\Gamma(\alpha'')}
\end{equation*} \label{lem:gamma2}
\end{lemma}

\begin{proof}
Consider allocation of $\{n_{ijk}\}_{k=1,...,{r_{i}}}$  items over the $r_i$ bins. There are two cases.
\begin{itemize}
    \item Let there be some index $k^\ast$ such that $n_{ijk^{\ast}} = n_{ij}$. This means that $n_{ijk}=0$ for all $k\neq k^\ast$. It follows that $\sum_{k=1}^{r_i} \log \frac{\Gamma(n_{ijk} + \alpha'')}{\Gamma(\alpha'')} = \log \frac{\Gamma(n_{ij} + \alpha'')}{\Gamma(\alpha'')}$.
    \item Let there be two indices $k_1$ and $k_2$ such that $n_{ijk_{1}} > 0$ and $n_{ijk_{2}} > 0$. Without loss of generality, we can assume that $n_{ijk_{1}} \geq n_{ijk_{2}}$. We move one item from bin $k_1$ to bin $k_2$. The sum $n_{ij}$ remains constant. By Lemma \ref{lem:gamma}, an increase in the RHS by $\log (n_{ijk_{1}}  + \alpha'') - \log (n_{ijk_{2}}  -1+ \alpha'')$, results in a corresponding increase in the LHS. Note that the assumption $n_{ijk_{1}} \geq n_{ijk_{2}}$ means that this increase is positive. 
     By increasing counts at the expense of small counts in  this way a sequence of distributions of the fixed sum $n_{ij}$ over the $r_i$ bins can be constructed for which the LHS of Lemma \ref{lem:gamma2} is increasing. The sequence terminates when $n_{ijk^{\ast}} = n_{ij}$ for some $k^\ast$. The result follows.
\end{itemize} 

\end{proof}

\begin{theorem}
\cite{cussens2015gobnilp}
\begin{align*}
     &\sum_{j=1}^{r_{\Pi_i}} \Bigg( \frac{\Gamma( \alpha')}{\Gamma(n_{ij}+\alpha')} +  \sum_{k=1}^{r_i} \log \frac{\Gamma(n_{ijk} + \frac{\alpha'}{r_i})}{\Gamma( \frac{\alpha'}{r_i})}\Bigg) \leq  \sum_{i=0, j:n_{ij}>0}^{n_{ij}} \log \Big( \frac{i+ a'/r_i}{i+\alpha}\Big).
\end{align*} \label{thm:gamma4}
\end{theorem}

\begin{proof}
 \begin{align*}
     &\sum_{j=1}^{r_{\Pi_i}} \Bigg( \log \frac{\Gamma( \alpha')}{\Gamma(n_{ij}+\alpha')} +  \sum_{k=1}^{r_i} \log \frac{\Gamma(n_{ijk} + \frac{\alpha'}{r_i})}{\Gamma( \frac{\alpha'}{r_i})}\Bigg)\\
      &\stackbin[]{1}{\leq} \sum_{j=1}^{r_{\Pi_i}}  \Bigg(\log \frac{\Gamma( \alpha')}{\Gamma(n_{ij}+ \alpha')} + \log \frac{\Gamma(n_{ij} + \frac{\alpha'}{r_i})}{\Gamma( \frac{\alpha'}{r_i})}\Bigg) \\
     &\stackbin[]{2}{\leq} \sum_{j=1}^{r_{\Pi_i}}  \Bigg(\log \frac{\Gamma( \alpha')}{\Gamma(n_{ij}+ \alpha')} \frac{\Gamma(n_{ij} + \frac{\alpha'}{r_i})}{\Gamma( \frac{\alpha'}{r_i})}\Bigg) \\
     & \stackbin[]{3}{\leq} \sum_{i=0, j:n_{ij}>0}^{n_{ij}-1} \Big( \log  \frac{i+ a'/r_i}{i+\alpha'}\Big)  \\
     &\stackbin[]{4}{\leq} \sum_{i=0, j:n_{ij}>0}^{n_{ij}} \log \Big( \frac{i+ a'/r_i}{i+\alpha'}\Big).
\end{align*}

Step 1 uses Lemma \ref{lem:gamma2}. Step 2 assumes $n_{ij} > 0$, and uses properties of the logarithm function. Step 3 uses Lemma \ref{lem:gamma}. The result follows as desired.
\end{proof}
\begin{corollary}\citep{cussens2015gobnilp} 
Given that $r_i^{+} := |\{j: n_{ij} > 0\}|$, then 
\begin{equation*}
         \sum_{j=1}^{r_{\Pi_i}} \log \frac{\Gamma( \alpha')}{\Gamma(n_{ij} + \alpha')} +  \sum_{k=1}^{r_i} \log \frac{\Gamma(n_{ijk} + \frac{\alpha'}{r_i})}{\Gamma( \frac{\alpha'}{r_i})}  \leq -r_i^{+} \log r_i.
\end{equation*}
\label{cor:gamma3}
\end{corollary}

\begin{proof}
 If $n_{ij} > 0$, then 
 \begin{align*}
     \sum_{i=0}^{n_{ij}} \log \Big( \frac{i+ a'/r_i}{i+\alpha'}\Big) &= - \log r_i \sum_{i=1}^{n_{ij}} \log \Big( \frac{i+ a'/r_i}{i+\alpha'}\Big) \leq -\log r_i.
 \end{align*}
Note that as $r_i \geq 2$, and $\alpha' > 0$, it is clear that $i + \alpha'/r_i < i+ \alpha'$. This means that each term in $\sum_{i=1}^{n_{ij}} \log \Big( \frac{i+ a'/r_i}{i+\alpha'}\Big)$ is negative.  This gives us the second inequality. The result then follows from Theorem \ref{thm:gamma4} as desired.
\end{proof}
\begin{corollary}\citep{cussens2015gobnilp}
Given a vertex variable $V_i$ and candidate parent sets $\Pi_i$ and $\Pi_i'$
such that $\Pi_i \subset \Pi_i'$ and $\Pi_i \neq \Pi_i'$, let
$r_i^{+}(\Pi_i')$ be the number of positive counts in the contingency table for $\Pi_i'$. If $\score{}{\Pi_i} < r_i^{+}(\Pi_i') \log r_i$ then $\Pi_i'$ and the supersets of $\Pi_i'$ can be safely pruned.
\label{cor:bdeujames}
\end{corollary}

 
 We generalize Corollary~\ref{cor:bdeujames} to Theorem~\ref{thm:jamesrelaxed} and prove it below.

\begin{proof}(Theorem~\ref{thm:jamesrelaxed})
 Let $G'$ be a Bayesian network where $\Pi_i'$ or one of its supersets is a parent set for $V_i$.  Let $G$ be another Bayesian network where $\Pi_i$ is the parent set for $V_i$. 
 
Consider the LHS of Corollary \ref{cor:gamma3}. It is the local BDeu score for a parent set $\Pi_i'$ which has $r_{\Pi_i}$ counts $n_{ij}$ in its contingency table and counts $n_{ijk}$ in the contingency table for $\Pi_i' \cup \{V_i\}$, where $\alpha' = \alpha/r_{\Pi_i}$ for some ESS $\alpha$. If $r_i^{+}$($\Pi_i') \log r_i > \score{}{\Pi_i} + \epsilon$ then $\score{}{\Pi_i} + \epsilon$ is
lower than the local BDeu score for $\Pi_i'$ due to Corollary \ref{cor:gamma3}. Take a candidate parent set $\Pi_i''$. If $\Pi_i' \subset \Pi_i''$ then $r_i^{+}(\Pi_i'')  \leq r_i^{+}(\Pi_i')$ and so $r_i^{+}(\Pi_i'') \log r_i \leq r_i^{+}(\Pi_i') \log r_i$, as $r_i \geq 2$. From this
it follows that the local score for $\Pi_i''$ must also be more than $\score{}{\Pi_i} + \epsilon $. Using Lemma \ref{lem:scoreprune}, the result follows as desired.
\end{proof}

\vskip 0.2in
\bibliography{csp,probabilistic,scip}

\end{document}